\RequirePackage{amsthm}

\documentclass[sn-mathphys-num,iicol]{sn-jnl}

\usepackage{graphicx}%
\usepackage{multirow}%
\usepackage{multicol}%
\usepackage{amsmath,amssymb,amsfonts}%
\usepackage{amsthm}%
\usepackage{mathrsfs}%
\usepackage[title]{appendix}%
\usepackage{xcolor}%
\usepackage{textcomp}%
\usepackage{manyfoot}%
\usepackage{booktabs}%
\usepackage{algorithm}%
\usepackage{algorithmicx}%
\usepackage{algpseudocode}%
\usepackage{listings}%
\usepackage{rotating}%
\usepackage{etoolbox}%
\usepackage{algpseudocode}%
\usepackage{url}%
\usepackage{natbib} 
\usepackage{boldline}
\usepackage{tabularx} 
\usepackage{float}
\usepackage{tabularx} 
\usepackage{multirow}
\usepackage[caption=false]{subfig}
\usepackage{graphicx} %
\graphicspath{{gfx/}} 
\usepackage{booktabs}
\usepackage{placeins}
\usepackage{tabularx}
\makeatletter
\def\hlinewd#1{%
\noalign{\ifnum0=`}\fi\hrule \@height #1 %
\futurelet\reserved@a\@xhline}
\makeatother

\newcommand{\mK}{\mathsf{K}}
\newcommand{\mF}{\mathsf{F}}

\newcommand{\mA}{\mathsf{A}}

\newcommand{\mY}{\mathsf{Y}}

\newcommand{\R}{\mathbb{R}}

\usepackage{float}
\usepackage{tabulary}
\usepackage{multirow}
\usepackage{subfig}
\usepackage{graphicx} %
\graphicspath{{gfx/}} 
\usepackage{booktabs}
\usepackage{placeins}
\usepackage{tabularx}
\makeatletter
\def\hlinewd#1{%
\noalign{\ifnum0=`}\fi\hrule \@height #1 %
\futurelet\reserved@a\@xhline}
\makeatother

\theoremstyle{thmstyleone}%
\newtheorem{theorem}{Theorem}
\newtheorem{corollary}{Corollary}

\newtheorem{proposition}[theorem]{Proposition}%

\theoremstyle{thmstyletwo}%
\newtheorem{remark}{Remark}%

\theoremstyle{thmstylethree}%
\newtheorem{definition}{Definition}%

\begin{document}

\title[Greedy feature selection]{Greedy feature selection: Classifier-dependent feature selection via greedy methods}



\author[1,2]{\fnm{Fabiana} \sur{Camattari}}\email{camattari@dima.unige.it}
\equalcont{These authors contributed equally to this work.}

\author[1,2*]{\fnm{Sabrina} \sur{Guastavino}}\email{guastavino@dima.unige.it}
\equalcont{These authors contributed equally to this work.}

\author[3]{\fnm{Francesco} \sur{Marchetti}}\email{francesco.marchetti@unipd.it}

\author[1,2]{\fnm{Michele} \sur{Piana}}\email{piana@dima.unige.it}

\author[4]{\fnm{Emma} \sur{Perracchione}}\email{emma.perracchione@polito.it}

\affil[1]{\orgdiv{MIDA, Dipartimento di Matematica}, \orgname{Università di Genova}, \orgaddress{\street{via Dodecaneso 35}, \city{Genova}, \postcode{16145}, 
\country{Italy}}}

\affil[2]{\orgdiv{Osservatorio Astrofisico di Torino}, \orgname{Istituto Nazionale di Astrofisica}, \orgaddress{\street{via Osservatorio 20}, \city{Pino Torinese, Torino}, \postcode{10025}, 
\country{Italy}}}

\affil[3]{\orgdiv{Dipartimento di Matematica "Tullio Levi-Civita"}, \orgname{Università di Padova}, \orgaddress{\street{via Trieste 63}, \city{Padova}, \postcode{35121}, 
\country{Italy}}}

\affil[4]{\orgdiv{Dipartimento di Scienze Matematiche \lq \lq Giuseppe Luigi Lagrange\rq }, \orgname{Politecnico di Torino}, \orgaddress{\street{Corso Duca degli Abruzzi 24}, \city{Torino}, \postcode{10129}, 
\country{Italy}}}


\abstract{The purpose of this study is to introduce a new approach to feature ranking for classification tasks, called in what follows greedy feature selection. In statistical learning, feature selection is usually realized by means of methods that are independent of the classifier applied to perform the prediction using that reduced number of features. Instead, greedy feature selection identifies the most important feature at each step and according to the selected classifier. In the paper, the benefits of such scheme are investigated theoretically in terms of model capacity indicators, such as the Vapnik-Chervonenkis (VC) dimension or the kernel alignment, and tested numerically by considering its application to the problem of predicting geo-effective manifestations of the active Sun.}




\maketitle

\section{Introduction}

Greedy algorithms are currently mainly used to iteratively select a reduced and appropriate number of examples according to some error indicators, and hence to produce surrogate and sparse models \cite{DUTTA2021110378,DSW,SH2017,Wenzel2021a,Wenzel2022a,WH2013}. The ambition of this paper is to analyze and extend greedy methods to work in the significantly more challenging case of feature reduction, i.e., as the computational core for feature-ranking schemes in the framework of classification issues. 

The importance of this application follows from the fact that sparsity enhancement is a crucial issue for statistical learning procedures, which might be performed, e.g., via Fisher score \cite{Fisher}, methods based on mutual information \cite{Peng}, Relief and its variants \cite{Robnik-Sikonja200323}. Indeed, supervised learning models are usually trained on a reduced number of features, which are typically obtained by means of either Lasso regression \cite{lasso} or
variations of the classical Lasso (Group Lasso \cite{yuan2006group}, Adaptive Lasso \cite{zou2006adaptive}, Adaptive Poisson re-weighted Lasso \cite{guastavino2019consistent} to mention a few) or linear Support Vector Machine (SVM) feature ranking \cite{Guyon}. However, at the state of the art, for a given classifier, these algorithms are often not able to actually capture all the corresponding most relevant features for that classifier. More specifically, in the case of Lasso and its generalizations \cite{freijeiro2022critical}, drawbacks in feature selection ability are shown when there
exists dependence structures among covariates. Therefore, here we designed feature-based greedy methods that iteratively select the most important feature at each step in a classifier-dependent fashion. We point out that any classifier can be used in this scheme, which allows a totally model-dependent feature ranking process.

At a more theoretical level, this study investigated the effectiveness of the greedy scheme in terms of the Vapnik-Chervonenkis (VC) dimension \cite{Vapnik71}, which is a complexity indicator common to any classifier, such as Feed-forward Neural Networks (FNNs), and it is related to the empirical risk \cite{Bartlett02}. As a particular instance, we further investigated how greedy methods behave for kernel-based classifiers, such as SVMs \cite{Taylor}, and in doing so we considered a particular complexity score, known as kernel alignment. These theoretical findings have been used for a case study concerning the classification and prediction of severe geomagnetic events triggered by solar flares.

Solar flares \cite{piana2022hard} are the most explosive manifestations of the active Sun and the main trigger of space weather \cite{schwenn2006space}. They may be followed by coronal mass ejections (CMEs) \cite{kahler1992solar}, which, in turn, may generate geomagnetic storms potentially impacting both space and on-earth technological assets \cite{gonzalez1994geomagnetic}. Data-driven approaches forecasting these events leverage machine learning algorithms trained against historical archives containing physical features extracted from remote-sensing data such as solar images or time series of physical parameters acquired from in-situ instruments \cite{bobra2015solar,camporeale2018machine,florios2018forecasting,2023ApJ...954..151G,telloni2023prediction}. These archives systematically provide a huge amount of descriptors and it is currently well-established that this redundancy of information significantly hampers the prediction performances of the forecasters \cite{campi2019feature}. Our feature-based greedy scheme was applied in this context, in order to identify among the features the redundant ones and consequently to improve the classification performances.

The paper is organized as follows. Section \ref{Greedy} introduces our greedy feature selection scheme, which will be motivated thanks to the theoretical analysis in Subsections \ref{VC} and \ref{SVM}. Section \ref{NE} describes the application of greedy feature selection to both simulated and real datasets. 
Our conclusions are offered in Section \ref{conclusion}.

\section{Greedy feature ranking schemes}
\label{Greedy}

Feature reduction (or feature subset selection) techniques can be classified into filter, wrapper, and embedded methods: filter methods identify an optimal subset based on general patterns in data \cite{bommert2020benchmark}; wrapper methods use a machine learning algorithm to search for the optimal subset by considering all possible feature combinations \cite{bajer2020wrapper}; in embedded methods 
feature selection is integrated or built into the classifier algorithm
\cite{zebari2020comprehensive}. This proposed greedy scheme falls in the class of wrapper feature subset selection methods, but unlike the classical approaches, such as recursive feature elimination (RFE) or recursive feature augmentation (RFA) \cite{Guyon} and forward step-wise selection \cite{StepwiseSelectionReference}, the proposed greedy method is fully model-dependent. 

Given a set of examples depending on several features, greedy methods are frequently used to find an optimal subset of examples and, for such task, since they might be target-dependent, they have already been proved to be effective (see e.g. \cite{Wenzel2021a,Wenzel2022a,doi:10.1137/23M1551201}). Here, instead of focusing on the examples, we drive our attention towards the problem of feature selection. To this aim, we considered a binary classification problem with training examples 
\begin{equation}\label{eq:data}
    \Xi = X \times Y = \{ (\boldsymbol{ x}_1, y_1),  \ldots ,(\boldsymbol{ x}_n, y_n)\},
\end{equation}
where $\boldsymbol{ x}_i \in \Omega \subseteq \mathbb{R}^d$ and $y_i \in   \mathbb{R}$. For the particular case of the binary classification setting, we fix $y_i \in \{-1, +1\}$. 

In the machine learning framework, feature reduction is typically performed by means of linear models, and once the features are identified, non-linear methods like neural networks are applied to predict the given task. However, the fact that some specific feature could be useful for some classifier does not imply that the same feature is relevant for any classification model, and this is probably the main weakness of current feature reduction methods in this context. Conversely, our feature-based greedy method (see e.g. \cite{Temlyakov2008} for a general overview of greedy methods) will consist in iteratively selecting the most important feature at each step and in agreement with the considered classifier.

To reach this objective, as usually done, we split the initial dataset $X$ into a training set, which consists of $\{(X^{(t)},Y^{(t)})\}$, and a validation set made of $\{(X^{(v)},Y^{(v)})\}$. Then, at the $k$-th greedy step, $k-1$ features have already been selected (without loosing generalities the first $k-1$) and we then train $d-k$ models ${\cal M}_p$ with $x_1,\ldots, x_{k-1},x_p$, $p=k, \ldots, d$. Then, given an accuracy score $\mu$ (the largest the better), we select the $k$-th feature as 
\begin{equation}\label{eq:1}
x_k = \arg \max_{p=k, \ldots, d} \mu({\cal M}_p(X^{(v)}),Y^{(v)}).     
\end{equation}

We point out that any model can be used in \eqref{eq:1}, and this implies a totally target-dependent feature selection, which also accounts for the model used to predict a given task.

In the following we investigate the effects of the proposed scheme in terms of VC dimension and for particular instances of kernel learning theory, while a stopping criterion for the algorithm is discussed later in view of the incoming analysis and trade-off remarks. 

\subsection{The VC dimension in the greedy framework}
\label{VC}

We consider the dataset \eqref{eq:data}, where we now suppose that $\Omega=\bigotimes_{k=1}^d\Omega^k$ with $\Omega^k=[a_k,b_k] \subset \mathbb{R}$. Given a classifying function $f:\Omega\longrightarrow Y$ we consider the \emph{zero-one loss function}
\begin{equation*}\label{01loss}
c(\boldsymbol{x},y,f)= \frac{1}{2}|f(\boldsymbol{x})-y|,
\end{equation*}
which is $0$ if $f(\boldsymbol{x})=y$ and $1$ otherwise. From this loss, we can define the {\em empirical risk}
\begin{equation*}\label{emprisk}
\hat{e}(\Xi,f)= \frac{1}{n}\sum_{i=1}^{n}{c(\boldsymbol{x}_i,y_i,f)}.
\end{equation*}
Assuming that $\Xi$ is sampled from some fixed unknown probability distribution $p(\boldsymbol{x},y)$ on $\Omega\times Y$, we note that the empirical risk is the empirical mean value of so-called {\em generalization risk}, i.e.:
\begin{equation*}\label{risk}
		e(f)= \int_{\Omega\times Y}{c(\boldsymbol{x},y,f)\:\mathrm{d}p(\boldsymbol{x},y)}, 
\end{equation*}
i.e., it is the mean value of $c$ averaged over all possible test samples generated by $p(\boldsymbol{x},y)$, and hence it represents the misclassification probability. However, minimizing the empirical risk does not necessarily correspond to a low generalization risk (refer, e.g., to \cite[\S 5]{Scholkopf02} or \cite[\S 5 \& \S 6]{Vapnik98}). Indeed, this might lead to poor generalization capability in the sense that statistical learning theory already proved that the generalization {\em capacity} of a given model is somehow inversely related to the empirical risk. Such general idea can be formalized in different ways, such as via the VC dimension. In order to define it, we need to introduce the concept of \textit{shattering} in this context. Let $\Xi_1,\dots,\Xi_{2^n}$ be all the different datasets obtainable taking all possible configurations of labels assigned to the data. A class $\mathcal{F}$ shatters the set $X$ if for every dataset $\Xi_i$, $i=1,\dots,2^n$, there exists a function $f:\Omega\longrightarrow Y$, $f\in\mathcal{F}$, such that $\hat{e}(\Xi_{i},f)=0$.

\begin{definition}\label{def:vc}
	The VC dimension of a class $\mathcal{F}$ of classifying functions is the largest natural number $s$ such that there exists a set $X$ of $s$ examples that can be shattered by $\mathcal{F}$. If such $s$ does not exist, then the VC dimension is $\infty$.
\end{definition}

Let us consider a class $\mathcal{F}$ of classifying functions on $\Omega$  whose VC dimension is $s<n$. Then, if $f\in\mathcal{F}$ and $\delta>0$, the bound
\begin{equation*}
		e(f)\le \hat{e}(\Xi,f)+C(s,n,\delta),
\end{equation*}
holds with probability $1-\delta$, where the so-called capacity term is
\begin{equation*}
		C(s,n,\delta)= \sqrt{\frac{1}{n}\bigg(s\bigg(\log{\frac{2n}{s}}+1\bigg)+\log{\frac{4}{\delta}}\bigg)}.
\end{equation*}
The generalization risk (and thus the test error) is bounded by the sum between the empirical risk (that is the training error) and the {capacity term} of the class, which {is monotonically increasing with the VC dimension}. If we choose a \emph{poor} class, we get a low VC dimension but a possibly high empirical risk; this situation is usually called \emph{underfitting}. On the other hand, by choosing a \emph{rich} class we can obtain a very small empirical risk, but the VC dimension, and thus the capacity term, is likely to be large; this condition is called \emph{overfitting}. In the following, our purpose is to study how the VC dimension evolves during the greedy steps. It is natural to guess that the capacity of a classifier increases if the information contained in an added feature is considered.
\begin{definition}
    Let $\mathcal{F}$ be a class of binary classifying functions $f:\Omega\longrightarrow Y$. Letting $\boldsymbol{e}_k$ be the $k$-th cardinal basis vector, we define the $k$-blind class $\mathcal{F}^{(k)}$, $k\in\{1,\dots,d\}$, $\mathcal{F}^{(k)}\subseteq \mathcal{F}$ as the class of functions $f^{(k)}:\Omega\longrightarrow Y$ such that
    \begin{equation*}
        f^{(k)}(\boldsymbol{x}) = f^{(k)}(\boldsymbol{x}+\delta\boldsymbol{e}_k),
    \end{equation*}
    for any $\delta\in\mathbb{R}$ such that $\boldsymbol{x}+\delta\boldsymbol{e}_k\in\Omega$.
\end{definition}
For example, consider the class of functions
\begin{equation}    
\mathcal{F}_{W,\boldsymbol{b}}:=\{f:\Omega\longrightarrow Y\:|\: f(\boldsymbol{x})=\tilde{f}(W\boldsymbol{x}+\boldsymbol{b})\},
\end{equation}
where $W$ is a $r\times d$ matrix and $\boldsymbol{b}$ is a $r\times 1$ vector, $r\ge 1$. Many well-known classifiers are included in $\mathcal{F}_{W,\boldsymbol{b}}$, such as, neural networks and linear models. In this setting, classifiers in $\mathcal{F}^{(k)}_{W,\boldsymbol{b}}$ can be constructed by restricting to $W$ and $\boldsymbol{b}$ such that $W_{:,k}=\mathbf{0}$, where $W_{:,k}$ is the $k$-th column of $W$, and $b_k=0$. 
\begin{remark}\label{rem:1}
	As $\mathcal{F}^{(k)}\subseteq \mathcal{F}$, the fact that that $\mathrm{VC}(\mathcal{F}^{(k)})\le \mathrm{VC}(\mathcal{F})$, trivially follows. 
\end{remark}
In order to formally prove that by adding a feature in the greedy step the obtained classifier cannnot be less expressive  (in terms of VC dimension) than the previous one, we introduce two maps:
\begin{itemize}
    \item 
    $\pi_{k}:\Omega\longrightarrow\bigotimes_{\substack{i=1\\i\neq k}}^d\Omega^i$, so that $\pi_{k}(\boldsymbol{x})=(x_1,\dots, x_{k-1},x_{k+1},\dots,x_{d})$, which is a projection.
    \item 
    $\iota_{\alpha}:\pi_{k}(\Omega)\longrightarrow\Omega$, $\alpha\in\Omega^k$, so that $\iota_{\alpha}(\boldsymbol{x})=(x_1,\dots, x_{k-1},\alpha,x_{k+1},\dots,x_{d})$, which is injective.
\end{itemize}
Note that applying $\iota_{\alpha}\circ\pi_k$ to $X$ has the effect of setting to $\alpha$ the $k$-th feature for all the examples.
\begin{proposition}\label{prop:1}
    $X$ is shattered by $\mathcal{F}^{(k)}$ if and only if $\iota_{\alpha}(\pi_{k}(X))$ is shattered by $\mathcal{F}^{(k)}$.
\end{proposition}
\begin{proof}
    Any classifier in $\mathcal{F}^{(k)}$ cannot rely on the $k$-th feature. Precisely, for each $\boldsymbol{x}_i\in X$ we can find $\delta_i\in\mathbb{R}$ so that $\boldsymbol{x}_i+\delta_i\boldsymbol{e}\in\iota_{\alpha}(\pi_{k}(X))$. Hence, it is equivalent for any function in $\mathcal{F}^{(k)}$ to shatter $X$ and $\iota_{\alpha}(\pi_{k}(X))$.
\end{proof}
For any function $f^{(k)}\in\mathcal{F}^{(k)}$ and $\alpha\in\Omega^k$, we can define a classifier $g:\pi_{k}(\Omega)\longrightarrow Y$ such that $g(\boldsymbol{x})=f^{(k)}(\iota_\alpha(\boldsymbol{x}))$. Denoting by $\mathcal{G}$ the class consisting of such functions $g$, we achieve the following result.
\begin{proposition}\label{prop:2}
    $\iota_{\alpha}(\pi_{k}(X))$ is shattered by $\mathcal{F}^{(k)}$ if and only if $\pi_{k}(X)$ is shattered by $\mathcal{G}$.
\end{proposition}
\begin{proof}
	Assume that there exists $f^{(k)}\in\mathcal{F}^{(k)}$ that shatters $\iota_{\alpha}(\pi_{k}(X))$. Note that the shattering does not rely on the $k$-th feature, which is constant, and therefore this is equivalent to shatter $\pi_{k}(\iota_{\alpha}(\pi_{k}(X)))=\pi_{k}(X)$ in a lower-dimensional space by means of a classifier $g$ so that $f^{(k)}=g\circ \pi_{k}$. Finally, by defining $\boldsymbol{x}^{(k)}=\pi_{k}(\boldsymbol{x})$, $\boldsymbol{x}\in \iota_{\alpha}(\pi_{k}(X))$, we further obtain $\boldsymbol{x}=\iota_{\alpha}(\boldsymbol{x}^{(k)})$, and therefore $g(\boldsymbol{x}^{(k)})=f^{(k)}(\iota_\alpha(\boldsymbol{x}^{(k)}))$ for $\boldsymbol{x}^{(k)}\in \pi_{k}(X)$, which completes the proof.
\end{proof}
\begin{corollary}
\label{corol:1}
    We have that $\mathrm{VC}(\mathcal{G})\le \mathrm{VC}(\mathcal{F})$.
\end{corollary}
\begin{proof}
	By putting together Propositions \ref{prop:1} and \ref{prop:2} we can affirm that $X$ is shattered by $\mathcal{F}^{(k)}$ if and only if $\pi_{k}(X)$ is shattered by $\mathcal{G}$. Note that $X$ and $\pi_{k}(X)$ have the same cardinality, and therefore $\mathrm{VC}(\mathcal{G})= \mathrm{VC}(\mathcal{F}^{(k)})$. We conclude the proof by virtue of Remark \ref{rem:1}.
\end{proof}

The results in Corollary \ref{corol:1} formalizes the idea that by adding a feature in the greedy step the obtained classifier cannot be less expressive than the previous one. Nevertheless, in this greedy context we consider a sort of trade-off that deals with the VC dimension: precisely, a high VC-dimension allows the model to fit more complex patterns but may lead to overfitting. Hence, we will discuss later a robust stopping criteria for the greedy iterative rule. Now, as a particular case study, we consider SVM classifiers, which are probably the most frequently used ones. Further, being they based on kernels, other capability measures concerning such classifiers can be straightforwardly studied.  

\subsection{SVM in the greedy framework}
\label{SVM}

Following the SVM literature, we drive our attention towards (strictly) positive definite kernels $\kappa: \Omega \times \Omega \longrightarrow \mathbb{R}$ that satisfy
\begin{equation*}
   \int_{\Omega} \kappa(\boldsymbol{x},\boldsymbol{z}) v(\boldsymbol{x}) v(\boldsymbol{z}) d\boldsymbol{x} d \boldsymbol{z} \geq 0, \quad \forall v \in L_2(\Omega),
\end{equation*}
for $\boldsymbol{x}, \boldsymbol{z} \in \Omega$. Then, those kernels can be decomposed via the Mercer's Theorem  as (see e.g. Theorem 2.2. \cite{Fasshauer} p. 107 or \cite{Mercer}):  
$$
\kappa(\boldsymbol{x}, \boldsymbol{z}) = \sum_{k \geq 0} \lambda_k \rho_k(\boldsymbol{x}) \rho_k(\boldsymbol{z}),\quad \boldsymbol{x}, \boldsymbol{z}\in\Omega,
$$
where $\{\lambda_k \}_{k \geq 0}$ are the (non-negative)
eigenvalues and $\{ \rho_k \}_{k \geq 0}$ are the ($L_2$-orthonormal) eigenfunctions of the operator 
$T: L_2(\Omega) \longrightarrow L_2(\Omega)$, given by
\begin{equation*}
T[v](\boldsymbol{x}) = \int_{\Omega} \kappa(\boldsymbol{x},\boldsymbol{z})  v(\boldsymbol{z})  d\boldsymbol{z}. 
\end{equation*}

Mercer's theorem provides an easy background for introducing feature maps and spaces.  Indeed, for Mercer kernels we can interpret the series representation in terms of an inner product in the so-called \emph{feature space} $F$, which is a Hilbert space. Indeed, we have that
\begin{equation}\label{perfeaturemap}
\kappa(\boldsymbol{x}, \boldsymbol{z}) = \langle \Phi(\boldsymbol{x}), \Phi(\boldsymbol{z}) \rangle_{F},\quad \boldsymbol{x}, \boldsymbol{z}\in\Omega,
\end{equation}
where $\Phi:\Omega \longrightarrow F$ is a \emph{feature map}. For a given kernel, the feature map and space are not unique. A possible solution is the one of taking the map $\Phi(\boldsymbol{x})=  \kappa(\cdot, \boldsymbol{x})$, which is linked to the characterization of $F$ as a reproducing kernel Hilbert space; see  \cite{Fasshauer15,Taylor} for further details. Both in machine learning literature and in approximation theory, radial kernels are truly common. They are kernels for whom there exists a Radial Basis Function (RBF) $ \varphi: \mathbb{R}_{+} \longrightarrow\mathbb{R}$, where $\mathbb{R}_{+}= [0,\infty)$, and (possibly) a shape parameter
$\gamma>0$ such that, for all $\boldsymbol{x},\boldsymbol{z}
\in \Omega$,
\begin{equation*}
\kappa(\boldsymbol{x},\boldsymbol{z})= \kappa_{\gamma}(\boldsymbol{x},\boldsymbol{z})=\varphi_{\gamma}( ||\boldsymbol{x}-\boldsymbol{z}||_2)=\varphi(r),
\end{equation*}
where $r=||\boldsymbol{x}-\boldsymbol{z}||_2$. Among all radial kernels, we remark that the Gaussian one is given by
\begin{equation}\label{gausskernel}
\kappa(\boldsymbol{x},\boldsymbol{z})= \kappa_{\gamma}(\boldsymbol{x},\boldsymbol{z})= {\rm e}^{- \gamma \|\boldsymbol{x}-\boldsymbol{z}\|_2^2}  ={\rm e}^{-\gamma r^2}.
\end{equation}
In the following, for simplicity, we omit the dependence on $\gamma$, which is also known as scale parameter in machine learning literature. 

With radial kernel as well, SVMs can be used for classification purposes and several complexity indicators, such as the kernel alignment, can be studied in order to have a better understanding of the greedy strategy based on SVM, i.e., when the generic classifier in Equation \eqref{eq:1} is an SVM function. The notion of kernel alignment was first introduced by \cite{NIPS2001_1f71e393} and later investigated in e.g. \cite{Wang}. Other common complexity indicators related to the alignment can be found in \cite{Donini17}. Given two kernels $\kappa_1$ and $\kappa_2: \Omega \times \Omega \longrightarrow \mathbb{R}^d$, the empirical alignment evaluates the similarity between the corresponding kernel matrices. It is given by 

\begin{equation*}
    \mA(X,\mK_1,\mK_2) = \dfrac{\left(\mK_1,\mK_2\right)_{\mF}}{\sqrt{||\mK_1||_{\mF} ||\mK_2||_{\mF}}},
\end{equation*}

where $\mK_1:=\mK_1(X)$ and $\mK_2:=\mK_2(X)$ denote the Gram matrices for the kernels $\kappa_1$ and $\kappa_2$ on $X$, respectively and 
\begin{equation*}
    \left(\mK_1,\mK_2\right)_{\mF} = \sum_{i,j=1}^n \kappa_1(\boldsymbol{x}_i,\boldsymbol{x}_j) \kappa_2(\boldsymbol{x}_i,\boldsymbol{x}_j).
\end{equation*}
The alignment can be seen as a similarity score based on the cosine of the angle. For arbitrary matrices, this score ranges between $-1$ and $1$. 

For classification purposes we can define an ideal target matrix as $\mY=\boldsymbol{y}\boldsymbol{y}^{\intercal}$, where
$\boldsymbol{y}=(y_1,\ldots,y_n)^{\intercal}$ is the vector of labels. Then the empirical alignment
between the kernel matrix $\mK$ and the target matrix $\mY$ can be written as:
\begin{equation*}
    \mA(X,\mK,\mY)= \dfrac{\left(\mK,\mY\right)_{\mF}}{\sqrt{||\mK||_{\mF} ||\mY||_{\mF}}}=\dfrac{\left(\mK,\mY\right)_{\mF}}{n \sqrt{||\mK||_{\mF} }}.
\end{equation*}
Such alignment with the target matrix is an indicator of the classification accuracy of a classifier. 
Indeed, to higher alignment scores correspond a separation of the data with a low bound on the generalization error \cite{Wang}. 

We now prove the following result which will be helpful in understanding our greedy approach. 
\begin{theorem}
    Given two kernels $\kappa_1$ and $\kappa_2: \Omega \times \Omega \longrightarrow \mathbb{R}^d$, if $||\mK_2||_{\mF} \geq ||\mK_1||_{\mF}$ then $\mA(X,\mK_1,\mY) \leq \mA(X,\mK_2,\mY)$. 
\end{theorem}
\begin{proof}
By hypothesis we have that:
    \begin{align*}
        \mA(X,\mK_1,\mY) & = \dfrac{\left(\mK_1,\mY\right)_{\mF}}{n \sqrt{||\mK_1||_{\mF} }} \leq \dfrac{\left(\mK1,\mY\right)_{\mF}}{n \sqrt{||\mK_2||_{\mF} }}.
        \end{align*}
        Then, by adding and subtracting ${\left(\mK_2,\mY \right)_{\mF}}$ at the numerator, 
        and thanks to the linearity of the norm, we obtain:
    \begin{align*}
        \mA(X,\mK_1,\mY) & \leq \dfrac{\left(\mK1,\mY\right)_{\mF}}{n \sqrt{||\mK_2||_{\mF} }}\\
         & =  \dfrac{\left(\mK1-\mK_2,\mY-\mY\right)_{\mF}}{n \sqrt{||\mK_2||_{\mF} }}+\dfrac{\left(\mK_2,\mY\right)_{\mF}}{n \sqrt{||\mK_2||_{\mF} }}\\
          & = \mA(X,\mK_2,\mY).
        \end{align*}
\end{proof}

Considering again Equation \eqref{eq:1}, let us denote by $X^{(k-1)}$ the dataset at the $(k-1)$ greedy step which already contains $k-1$ features and by $X^{(k)}$ the one that is constructed at the $k$-th step accordingly to our greedy rule. Then, as a corollary of the previous theorem, we have the following result. 
\begin{corollary}
   If $\kappa$ is a non-increasing radial kernel, then $$\mA(X^{(k)},\mK(X^{(k)}),\mY) \geq \mA(X^{(k-1)},\mK(X^{(k-1)}),\mY).$$
\end{corollary}
\begin{proof}
Being $\varphi: \mathbb{R}_{+} \longrightarrow \mathbb{R}$ non-increasing, for $$\boldsymbol{x}, \boldsymbol{z} \in \mathbb{R}^{d},$$ we obtain

\begin{align*}
&\varphi\left(\|\boldsymbol{x}-\boldsymbol{z}\|_2\right)  = \\ 
&=\varphi(\|(x_1,x_2,\ldots,x_k)- 
(z_1,z_2,\ldots,z_k)\|_2) \leq \\
&\leq \varphi(\|(x_1,x_2,\ldots,x_{k-1})- 
(z_1,z_2,\ldots,z_{k-1})\|_2),
\end{align*}which in particular implies that
$$
\mK_{ij}(X^{(k-1)}) \geq \mK_{ij} (X^{(k)})\geq 0, \quad i,j=1, \ldots, n.
$$
Thus, we get
$$
\| \mK (X^{(k-1)}) \|_{\mF} \geq \| \mK (X^{(k)})  \|_{\mF}, 
$$
and hence $$\mA(X^{(k)},\mK (X^{(k)}),\mY) \geq \mA(X^{(k-1)},\mK (X^{(k-1)}),\mY).$$
\end{proof}

Note that this kind of feature augmentation strategy via greedy schemes shows some similarities with the so-called Variably Scaled Kernels (VSKs), first introduced in \cite{Bozzini1} and recently applied in the framework of inverse problems, see e.g. \cite{Perracchione_2023,perracchione_2021}. Indeed, both approaches are based on adding features and both are again characterized by a trade-off between the model capacity, which can be characterized by the kernel alignment, and the model accuracy. To achieve a good trade-off between these two factors we need a stopping criteria for the iterative rule shown in \eqref{eq:1}.

\subsection{Stopping criterion}

In actual applications, the greedy iterative algorithm should select, at first, the most relevant features, and then, if no relevant features are available, any accuracy score should saturate. Among several scores \( \mu \), a robust one is the so-called True Skill Statistic (TSS) for its characteristic of being insensitive to class imbalance \cite{bloomfield2012toward}. Precisely, letting TN, FP, FN, TP respectively the number of true negatives, false positives, false negatives and true positives, the TSS is
defined by: 
\begin{align*}
        {\rm TSS}({\rm TN,FP,FN,TP})= \; {\rm recall}({\rm TN,FP,FN,TP}) \nonumber \\
        + {\; \rm  specificity}({\rm TN,FP,FN,TP})-1, 
\end{align*}
where 
\begin{equation}\label{eq:recall_spec}
    {\rm recall}({\rm TN,FP,FN,TP}) = \dfrac{{\rm TP}}{{\rm FN+TP}},
\end{equation}and
\begin{equation}\label{eq:spec}
    {\rm  specificity} ({\rm TN,FP,FN,TP}) = \dfrac{{\rm TN}}{{\rm FP+TN}}.
\end{equation}

In order to introduce a stopping criteria, we need to point out that we construct a greedy feature ranking by considering, at each step, $q$ splits of the dataset into training and validation sets. Specifically, at the \(k\)-th step of the greedy algorithm, each one of the \(d-k\) datasets, composed by the \( k-1 \) selected features \(x_1,\dots,x_{k-1}\) and the added one \(x_p\) ($p=k, \ldots, d$), is divided into training and validation sets, namely $\{(X_{p,h}^{(t)},Y_{p,h}^{(t)})\}$ and $\{(X_{p,h}^{(v)},Y_{p,h}^{(v)})\}$ respectively, for $h=1,\dots,q$, and fixed $p$.
Hence, once the models ${\cal M}_p$ ($p=k, \ldots, d$) have been trained the \(k\)-th feature is chosen  so that:
\begin{equation}
x_k = \arg \max_{p=k, \ldots, d} \frac{1}{q}\sum_{h=1}^q \mu({\cal M}_p(X_{p,h}^{(v)}),Y_{p,h}^{(v)}), 
\end{equation}
where $\mu$ is the TSS score. 
Then, letting $m_k$ be the average of the TSS scores computed on different splits at the $k$-th step and $\sigma_k$ the associated standard deviation, we stop the greedy iteration at the $k$-th step if:
 \begin{equation}\label{stop}
    \dfrac{|m_{k+1}-m_{k}|}{\sqrt{(\sigma_{k+1}^2+\sigma_{k}^2)} }< \tau,
\end{equation}
and $\tau$ is a given threshold. 
By doing so, we stop the greedy algorithm when the added feature does not contribute to the accuracy score. In order to better understand this fact, we provide in the following a numerical experiment with synthetic data. Dealing with real data, we might stop the greedy iteration as shown in \eqref{stop}, but then select only the first $k^*$ features, where $k^*$ is 
\begin{equation}\label{stop1}
    k^* =\textrm{arg} \max_{j=1,\ldots,k} m_j.
\end{equation}

\section{Numerical experiments}
\label{NE}

The first numerical experiment wants to numerically show the convergence of the greedy algorithm and the efficacy of the stopping rule. Then, we will show an application in the context of space weather, which aims to show how this general method is able to infer on the physical aspects of the problem.  

\subsection{Applications to a toy dataset}\label{simulated}

We first focused on the application of the non-linear SVM greedy technique to a balanced simulated dataset constructed as follows: we considered the set \(X = \{\boldsymbol{x}_i\}_{i=1}^n 
\) of \(n=1000\) random points in dimension \(d=15\) sampled from a uniform distribution over \([0,1)\) and the set of corresponding function values \(\{f_{\alpha,i}=f_{\alpha}(\boldsymbol{x}_i)\}_{i=1}^n\), where \(f_{\alpha}: [0,1)^d \longrightarrow \R\) is defined as 

\begin{equation}\label{testfunction}
\begin{aligned}
    f_{\alpha}(\boldsymbol{x})=   e^{x_1^2}+e^{x_2}+3x_3+2\cos{(x_4x_5)}\\
    +4x_6^2 
    +10^{{\alpha}}\sum_{j=7}^{d} x_j.
\end{aligned}
\end{equation}
and \(\alpha\in\{-8,-6,-4,-2\}\). Each \(f_{\alpha,i}\) was then labeled according to a threshold value to obtain the set of outputs \(Y=\{y_i\}\), i.e., \(y_i = 1\) if \(f_{\alpha,i}\) is greater than the mean value attained by \(f_\alpha\)
, and \(y_i=-1\) otherwise. From \eqref{testfunction} we note that the first 6 features (i.e., \(x_j\) for \(j=1,\dots,6\)) are meaningful for classification purposes when $\alpha$ is lower than $-4$, while the contribution of the remaining ones is negligible. The classifier used in the following was an SVM model for which both the scale parameter of the Gaussian kernel and the bounding box are optimized via standard cross-validation.  The results of using such a classifier into the greedy scheme are reported in Table \ref{tab:simulated_rankings}. Such table contains the greedy ranking of the features \(x_j\), \(j=1,\dots ,d\), and the TSS values obtained at each step by averaging over 7 different validation sets. Letting \(\tau=9 {\rm e}-2\) be the threshold for the stopping criteria in \eqref{stop}, the greedy algorithm selected the features reported in Table \ref{tab:simulated_rankings}, which are above the black solid line. As expected, the algorithm selected only the first six features (the most relevant ones) when $\alpha$ is small enough ($\alpha \leq -6 $). Then, as soon as the remaining features become more meaningful the greedy selection takes into account more features. In this didactic example we report all the TSS values until the end, to emphasise the robustness of our procedure that correctly identified the most relevant features.

\begin{table*}[ht!]
\caption{Feature ranking for the greedy scheme on the  dataset generated as in \eqref{testfunction}. The selected features are identified by the bold line in the table.}
\label{tab:simulated_rankings}
    \centering
    
    \begin{tabular}{|cc|cc|cc|cc|}
        \hline
         \multicolumn{2} {|c|} {\(\alpha=-8\)}
        & \multicolumn{2} {c|} {\(\alpha=-6\)}
        & \multicolumn{2} {c|} {\(\alpha=-4\)}
        & \multicolumn{2} {c|} {\(\alpha=-2\)}
         \\
        
        \hline
        $x_j$  & TSS & $x_j$  & TSS & $x_j$  & TSS & $x_j$  & TSS\\
        \hline
        {\(x_1\)} & 
        0.204 
        \(\pm\) 0.050 
        & {\(x_1\)} & 
        0.204 
        \(\pm\) 0.050 
        & {\(x_1\)} & 
        0.198 
        \(\pm\) 0.048 
 
        & {\(x_1\)} & 
        0.197 
        \(\pm\) 0.034 
\\
        {\(x_6\)} & 
        0.550 
        \(\pm\) 0.049 
        & {\(x_6\)} & 
        0.553 
        \(\pm\) 0.050 

        & {\(x_6\)} & 
        0.558 
        \(\pm\) 0.048 
        & {\(x_6\)} &
        0.553 
        \(\pm\) 0.044 
\\
        {\(x_3\)} & 
        0.798 
        \(\pm\) 0.049 
        & {\(x_3\)} & 
        0.798 
        \(\pm\) 0.049 
        
        & {\(x_3\)} & 
        0.800 
        \(\pm\) 0.051 
        & {\(x_3\)} &
        0.787 
        \(\pm\) 0.041 
\\
        {\(x_2\)} & 
        0.930 
        \(\pm\) 0.030 
        & {\(x_2\)} & 
        0.930 
        \(\pm\) 0.030 
        
        & {\(x_2\)} & 
        0.933 
        \(\pm\) 0.031 
        & {\(x_2\)} &
        0.888 
        \(\pm\) 0.017 
\\
        {\(x_4\)} & 
        0.939 
        \(\pm\) 0.021 
        & {\(x_4\)} & 
        0.939 
        \(\pm\) 0.021 
 
        & {\(x_4\)} & 
        0.939 
        \(\pm\) 0.024 
        & {\(x_4\)} &
        0.895 
        \(\pm\) 0.025 
\\
        {\(x_5\)} & 
        0.954 
        \(\pm\) 0.015 
        
        & {\(x_5\)} & 
        0.954 
        \(\pm\) 0.015 
        & {\(x_5\)} &
        0.961 
        \(\pm\) 0.017 
        & {\(x_{13}\)} &
        0.899 
        \(\pm\) 0.024 
\\
\clineB{1-4}{4}     
        \(x_{12}\) & 
        0.953 
        \(\pm\) 0.014 

        
        & \(x_{12}\) &
        0.953 
        \(\pm\) 0.014 
 
        & \(x_{12}\) & 
        0.953 
        \(\pm\) 0.014 
        & \(x_{8}\) & 
        0.888 
        \(\pm\) 0.034 
\\
        \(x_{13}\) & 
        0.953 
        \(\pm\) 0.022 

        
        & \(x_{13}\) & 
        0.953 
        \(\pm\) 0.022 
 
        & \(x_{9}\) & 
        0.948 
        \(\pm\) 0.022 
        & \(x_{5}\) & 
        0.895 
        \(\pm\) 0.036 
\\
        
        \clineB{5-6}{4}
        \(x_{9}\) & 
        0.946 
        \(\pm\) 0.028 

        
        & \(x_{9}\) & 
        0.946 
        \(\pm\) 0.028 
        & \(x_{13}\) & 
        0.948 
        \(\pm\) 0.025 
        & \(x_{14}\) &
        0.900 
        \(\pm\) 0.035 
\\
\clineB{7-8}{4}
        \(x_{11}\) & 
        0.928 
        \(\pm\) 0.039 

        
        & \(x_{11}\) & 
        0.928 
        \(\pm\) 0.039 
        
        & \(x_{14}\) & 
        0.929 
        \(\pm\) 0.026 
        
        & 
        \(x_{7}\) &
        0.896 
        \(\pm\) 0.039 
\\
 
        \(x_{14}\) & 
        0.932 
        \(\pm\) 0.021 

        
        & \(x_{14}\)  & 
        0.932 
        \(\pm\) 0.021 
        
        & \(x_{11}\) & 
        0.920 
        \(\pm\) 0.027 
        & \(x_{10}\) &
        0.895 
        \(\pm\) 0.039 
\\
        \(x_{7}\) & 
        0.914 
        \(\pm\) 0.023 

        
        & \(x_{7}\)  & 
        0.914 
        \(\pm\) 0.023 
        
        & \(x_{7}\) & 
        0.911 
        \(\pm\) 0.031 
        & \(x_{12}\) &
        0.881 
        \(\pm\) 0.036 
\\
        \(x_{10}\) & 
        0.889 
        \(\pm\) 0.024 
        
        & \(x_{10}\) & 
        0.889 
        \(\pm\) 0.024 
        
        & \(x_{8}\) & 
        0.890 
        \(\pm\) 0.043 
        & \(x_{11}\) &
        0.881 
        \(\pm\) 0.046 
\\
        \(x_{8}\) &  
        0.883 
        \(\pm\) 0.025 
        
        & \(x_{8}\) & 
        0.873 
        \(\pm\) 0.041 
        
        & \(x_{10}\) & 
        0.878 
        \(\pm\) 0.027 
        & \(x_{9}\) &
        0.871 
        \(\pm\) 0.028 
        \\
        \hline
    \end{tabular}

\end{table*}

\subsection{Applications to solar physics: geo-effectiveness prediction}\label{flares}
We now focus on a significant space weather application, i.e., the prediction of severe geo-effectiveness events based on the use of both remote sensing and in-situ data. More specifically, data-driven methods addressing this task typically utilizes features acquired by in-situ instruments at Lagrangian point L1 (i.e., the Lagrangian point between the Sun and the Earth) to forecast a significant increase of the SYM-H index, i.e., the expression of the geomagnetic disturbance at Earth \cite{wanliss2006high}.

\subsubsection{The dataset and the models}
The dataset we used consisted of a collection of solar wind, geomagnetic and energetic indices. In particular, it was composed by  
$N=7888320$ examples and $d=15$ features sampled at each minute starting from  (1-st January 2005) to  (31-st December 2019). Below we summarize the features we used:
\begin{enumerate}
    \item B [nT], the magnetic field intensity, and B$_{\rm x}$, B$_{\rm y}$ and  B$_{\rm z}$  [nT], its three coordinates. 
    \item V [Km/s], the velocity of the solar wind, and V$_{\rm x}$, V$_{\rm y}$ and V$_{\rm z}$ [Km/s], its three coordinates. 
    \item T, the proton temperature, and $\rho$, the proton density number [cm$^{-3}$].
    \item E$_{\rm k}$,  E$_{\rm m}$,  E$_{\rm t}$ the kinetic,  magnetic and total energies. 
    \item H$_{\rm m}$, the magnetic helicity. 
    \item SYM-H [nT], a geomagnetic activity index that quantifies the level of geomagnetic disturbance.
\end{enumerate}
The first ten features were acquired at the Lagrangian point L1 by in-situ instruments, the energies and the magnetic helicity being adimensional derived quantities, and the SYM-H being measured at Earth. The task considered in what follows consisted in identifying the most relevant features used to predict whereas a geo-effective event occurred, i.e., when the SYM-H was less than $-50$ nT (label 1), or not (label -1). The dataset at our disposal was highly unbalanced: the rate of positive events was about 2.5$\%$. In order exploit our data analysis, we first need to fix the notation. We denote by $\tilde{X} = \{\tilde{\boldsymbol{x}_i} \}_{i=1}^N \subseteq \Omega$, where $\Omega \subseteq \mathbb{R}^d$, the set of input samples an by $\tilde{Y}=\{\tilde{y}_i \}_{i=1}^N$, with $\tilde{y}_i \in \{-1,1\}$, the set of associated labels. The features denoted by $\tilde{{x}_j}$, $j=1,\ldots,d$, represent respectively B, B$_{\rm{x}}$, B$_{\rm{y}}$, B$_{\rm{z}}$, V, V$_{\rm{x}}$, V$_{\rm{y}}$, V$_{\rm{z}}$, T, $\rho$, E$_{\rm{k}}$, E$_{\rm{m}}$, E$_{\rm{t}}$, H$_{\rm{m}}$ and the SYM-H. 

The analysis was performed with data aggregated by hours, i.e., letting $m=60$, $n =N/m$ and $$\boldsymbol{x}_i=\dfrac{(\sum_{k=i}^{i+m} \tilde{\boldsymbol{x}_k})}{m},$$ we focused on ${X} = \{{\boldsymbol{x}_i} \}_{i=1}^n \subseteq \Omega$. Similarly, we defined the set of aggregated labels ${Y}=\{{y}_i \}_{i=1}^n$. 

Given $X$ and $Y$, the first step of our study consisted in using different feature selection approaches to rank the features accordingly to their relevance (see Subsection \ref{greedysolar}). After this step, we investigated how these results can be exploited to improve the prediction task (see Subsection \ref{predictionsolar}). In doing so, we used both SVM and a Feed-forward Neural Network (FNN) in order to 
predict whether a geo-effective event occurred or not in the next hour. Specifically, the SVM algorithm was trained by performing a randomized and cross-validated search over 
the hyper-parameters of the model (the regularization parameter \(C\) and the kernel  coefficient \(\gamma\))
taken from uniform distributions on \(I_C=[0.1, 1000]\) and \(I_{\gamma}=[0.001, 0.1]\) respectively. 
Instead, the FNN architecture was characterized by  $7$ hidden layers. The Rectified Linear Unit (ReLU) function was used to activate the hidden layers, the sigmoid activation function was applied to activate the output, and the binary cross-entropy was used as loss function. The model was trained over $200$ epochs using the Adam optimizer with learning rate equal to $0.001$, with a mini-batch size of $64$. In order to prevent overfitting, an $L^2$ regularization constraint was set as $0.01$ in the first two layers. Further, we used an early stopping strategy to select the best epoch with respect to the validation loss.

\subsubsection{Greedy feature selection approaches}
\label{greedysolar}
In order to apply efficiently our greedy strategy to both SVM and FNN, we first considered a subset \( X_p \) of the original dataset \( X \) with a reduced number of examples: we took \( p= 3333 \) examples. 
The so-constructed ranking was compared to a state-of-the-art method, i.e., the Lasso feature selection. Precisely, the active set of features returned by Lasso was composed by: B$_{\rm{x}}$, B$_{\rm{y}}$, B$_{\rm{z}}$, V$_{\rm{y}}$, V$_{\rm{z}}$, T, $\rho$, E$_{\rm{k}}$, E$_{\rm{m}}$, E$_{\rm{t}}$, H$_{\rm{m}}$ and the SYM-H. Note that neither V and B, which are physically meaningful for the considered task, were selected by cross-validated Lasso. 

In Table \ref{tab:greedy_rankings} we report the results of the greedy feature ranking scheme by using SVM and FNN. In this table, the features are ordered accordingly to the greedy selection. In particular, the greedy iteration stopped with all the features reported in the table accordingly to \eqref{stop}, but the selected features were only the ones above the bold line, as in \eqref{stop1}. We can note that, the features selected for both SVM and FNN are only a few, and this is due to the fact that greedy schemes are model-dependent and hence are able to truly capture the most significant ones. We further point out that in order to extract such features, we made use of a validation set and we did not considered any test set, since it was not at our disposal. Therefore, the greedy feature extraction is coherently based on the TSS computed on the validation set, and not on the test set. Nevertheless, we are now interested in understanding how the selected features work in the prediction (on tests sets) of the original task and with all examples.

 
\begin{table}[h!]
    \caption{Feature rankings for the greedy schemes on the dataset used for the prediction of geo-effective solar storms.}
    \label{tab:greedy_rankings}
    \centering
    \begin{tabular}{|cc|cc|}
        \hline
         \multicolumn{2} {|c|} {Greedy ranking (SVM)}
         & \multicolumn{2} {c|} {Greedy ranking (FNN)}
         \\      
        \hline
        $x_j$  & TSS & $x_j$  & TSS\\
        \hline
        {SYM-H} & 0.703 
        \(\pm\) 0.179 
        
        & {SYM-H} & 0.936 
        \(\pm\) 0.052 
        \\
        {B$_{\rm{z}}$} & 0.823 
        \(\pm\) 0.121 
        
        & {B} & 0.943 
        \(\pm\) 0.034 
        \\
        {V} & 0.804 
        \(\pm\) 0.115 
        & {E$_{\rm{t}}$} & 0.958 
        \(\pm\) 0.039 
        
        \\
        \clineB{3-4}{4}
        {E$_{\rm{t}}$} & 0.825 
        \(\pm\) 0.176 
        
        & V$_{\rm{x}}$ & 0.934 
        \(\pm\) 0.078 
        \\
        {V$_{\rm{x}}$} & 0.853 
        \(\pm\) 0.147 
        
        & & \\

        \clineB{1-2}{4}
        
        E$_{\rm{m}}$ & 0.804 
        \(\pm\) 0.184 
        & & \\
        B & 0.835 
        \(\pm\) 0.115 
        & & \\
        \hline
    \end{tabular}
\end{table} 

Interestingly, the features extracted as the most prominent ones are indeed those associated with physical processes involved in the transfer of energy from the CMEs to the Earth's magnetosphere and, thus, with the CME likelihood for inducing geomagnetic storms. B$_{\rm{z}}$, i.e., a southward directed interplanetary magnetic field, is indeed required for magnetic reconnection with the Earth's magnetic field to occur, and thus for the energy carried by the solar wind and/or CMEs to be transferred to the Earth system. In addition, the bulk speed V, or equivalently the radial component of the flow velocity vector V$_{\rm{x}}$, is directly related to the kinetic energy of the solar wind. On the one hand, it is well known that particularly fast particle streams or solar transients can compress the magnetosphere on the sunward side. On the other hand, high levels of magnetic energy (quadratically proportional to the magnetic field intensity) can be converted into thermal energy that heats the Earth's atmosphere, expanding it. In both cases, it appears evident that the transfer of energy, either kinetic or magnetic or total, enabled by the magnetic reconnection between the interplanetary and terrestrial magnetic fields, disrupts the magnetosphere current system, thus causing geomagnetic disturbances. As a conclusion, the extracted features are the physical quantities with the higher expected  predictive capability.

\subsubsection{Prediction of geo-effective solar events with greedy-selected features}
\label{predictionsolar}
In order to numerically validate our greedy procedure we compared the performances of SVM and FNN trained with respectively: all features, the features returned by Lasso, and the greedily selected features. The comparison was performed by computing several scores (reported in Tables \ref{tab:prediction_SVM} and \ref{tab:prediction_FNN}) and by averaging on different splits of the test set: in particular, we computed the TSS as reference score, the Heidke Skill Score (HSS) \cite{Heidke1926}, precision, recall (see equation \eqref{eq:recall_spec}), specificity (see equation \eqref{eq:spec}), F1 score (which is the  harmonic mean of precision and recall), and balanced accuracy (which is the
arithmetic mean between recall and specificity). We can observe that for the SVM-based prediction, when using the features extracted with the greedy procedure, we have a remarkable improvement of all accuracy scores. Further, although the performances of the FNN are essentially the same, independently of the feature selection scheme, we note that we were able to achieve the same accuracy scores with only a few features selected ad hoc (3 in this case). This points out again the fact that features extracted by methods, such as Lasso, might be redundant for the considered classifiers. This is even more evident when using the FNN algorithm, which achieved the same accuracy with only 3 greedily selected features. The improvement in terms of accuracy was remarkable only for SVM classifiers, which is known to be less robust then neural networks to \emph{noise}, i.e., redundant information stored in redundant features.

\begin{table*}
\caption{Average scores obtained with SVM using different subsets of features. 
}
\label{tab:prediction_SVM}
    \centering
    \begin{tabular}{|c|c|c|c|}
    \clineB{1-4}{0.5}
    \multicolumn{1} {|c} {Metric} & \multicolumn{1} {|c|} {All features} & \multicolumn{1} {c|} {LASSO selection} & \multicolumn{1} {c|} {Greedy selection (SVM)} \\
    \hline
    TSS & 0.679 
    \(\pm\) 0.055 
    & 0.677 
    \(\pm\) 0.088 
    & 0.736 
    \(\pm\) 0.051 
    \\
    HSS & 0.731 
    \(\pm\) 0.043 
    & 0.739 
    \(\pm\) 0.040 
    & 0.808 
    \(\pm\) 0.021 
    \\
    Precision & 0.822 
    \(\pm\) 0.117 
    & 0.840 
    \(\pm\) 0.068 
    &  0.909 
    \(\pm\) 0.043 
    \\
    Recall & 0.683 
    \(\pm\) 0.059 
    & 0.681 
    \(\pm\) 0.090 
    &  0.738 
    \(\pm\) 0.052 
    \\
    Specificity & 0.995 
    \(\pm\) 0.005 
    & 0.996 
    \(\pm\) 0.002 
    & 0.998 
    \(\pm\) 0.001 
    \\
    F1 score & 0.737 
    \(\pm\) 0.041 
    & 0.745 
    \(\pm\) 0.039 
    & 0.812 
    \(\pm\) 0.021 
    \\
    Balanced Accuracy & 0.839 
    \(\pm\) 0.027 
    & 0.839 
    \(\pm\) 0.044 
    & 0.868 
    \(\pm\) 0.026 
    \\
    \hline
    \end{tabular}
\end{table*}

\begin{table*}
\caption{Average scores obtained with FNN using different subsets of features.}
\label{tab:prediction_FNN}
    \centering
    \begin{tabular}{|c|c|c|c|}
    \clineB{1-4}{0.5}
    \multicolumn{1} {|c} {Metric} & \multicolumn{1} {|c|} {All features} & \multicolumn{1} {c|} {LASSO selection} & \multicolumn{1} {c|} {Greedy selection (SVM)} \\
    \hline
    TSS & 0.913 
    \(\pm\) 0.054 
    & 0.917 
    \(\pm\) 0.043 
    & 0.895 
    \(\pm\) 0.054 
    \\
    HSS & 0.685 
    \(\pm\) 0.105 
    & 0.638 
    \(\pm\) 0.119 
    & 0.669 
    \(\pm\) 0.128 
    \\ 

    Precision & 0.577 
    \(\pm\) 0.153 
    & 0.519 
    \(\pm\) 0.159 
    
    & 0.571 
    \(\pm\) 0.176 
    \\
    Recall & 0.935 
    \(\pm\) 0.065 
    & 0.945 
    \(\pm\) 0.056 
    & 0.919 
    \(\pm\) 0.068 
    
    \\
    Specificity & 0.978 
    \(\pm\) 0.014 
    & 0.972 
    \(\pm\) 0.017 
    & 0.976 
    \(\pm\) 0.019 
    \\
    F1 score & 0.695 
    \(\pm\) 0.010 
    & 0.650 
    \(\pm\) 0.114 
    & 0.680 
    \(\pm\) 0.122 
    \\
    Balanced Accuracy & 0.957 
    \(\pm\) 0.027 
    & 0.959 
    \(\pm\) 0.022 
    & 
    0.948 
    \(\pm\) 0.027 
    \\
    \hline
    \end{tabular}

\end{table*}

\section{Conclusions and future work}
\label{conclusion}
We introduced a novel class of feature reduction schemes, namely greedy feature selection algorithms. Their main advantage consists in the fact that they are able to identify the most relevant features for any given classifier. We studied their behaviour both analytically and numerically. Analytically, we could conclude that the models constructed in such a way cannot be less expressive than the standard ones (in terms of VC dimension or kernel alignment). Numerically, we showed their efficacy on a problem associated to the prediction of geo-effective events of the active Sun. As the activity of the Sun is cyclic, work in progress consists in using greedy schemes to study which features are relevant on either high or low activity periods. Finally, as there is a growing interest in physics-informed neural networks (PINN), we should investigate, both theoretically and numerically, which are the challenges that greedy methods could achieve in this context.

\section*{Acknowledgements}
Fabiana Camattari and Emma Perracchione kindly acknowledge the support of the Fondazione Compagnia di San Paolo within the framework of the Artificial Intelligence Call
for Proposals, AIxtreme project (ID Rol: 71708). Sabrina Guastavino was supported by the Programma Operativo Nazionale (PON) “Ricerca e Innovazione” 2014–2020. The research by Michele Piana was supported in part by the MIUR Excellence Department Project awarded to Dipartimento di Matematica, Università di Genova, CUP D33C23001110001. All authors are members of the Gruppo Nazionale per il Calcolo Scientifico - Istituto Nazionale di Alta Matematica (GNCS - INdAM).

\bibliography{sn-bibliography}

\end{document}